\documentclass[conference]{IEEEtran}
\pdfoutput=1
\usepackage[noresetcount,lined,boxed, linesnumbered]{algorithm2e}
\usepackage{algpseudocode}
\usepackage{cite}
\usepackage{comment}
\usepackage{amsthm}
\usepackage[utf8]{inputenc} 
\usepackage[T1]{fontenc}    
\usepackage{hyperref}       
\usepackage{url}            
\usepackage{booktabs}       
\usepackage{amsfonts}       
\usepackage{amssymb}
\usepackage{amsmath}
\usepackage{nicefrac}       
\usepackage[mathcal]{eucal}
\usepackage{microtype}      
\usepackage{lipsum}
\usepackage{graphicx}
\usepackage{subcaption}
\usepackage{float}
\usepackage{svg}
\usepackage{makecell}
\usepackage{grffile}

\graphicspath{ {./figures/} }

\newtheorem{theorem}{Theorem}[section]

\newtheorem{proposition}{Proposition}
\newtheorem{lemma}[theorem]{Lemma}
\newtheorem*{remark}{Remark}

\usepackage{textcomp}
\usepackage{xcolor}

\def\BibTeX{{\rm B\kern-.05em{\sc i\kern-.025em b}\kern-.08em
    T\kern-.1667em\lower.7ex\hbox{E}\kern-.125emX}}

\usepackage{graphicx}
\usepackage{epstopdf}
\usepackage{hyperref}
\usepackage{amsmath}
\usepackage{amsfonts}
\usepackage{latexsym}
\usepackage{color}
\usepackage{listings}
\usepackage{placeins}
\usepackage{subcaption}
\usepackage[labelfont=bf]{caption}

\usepackage{enumitem}
\setlist[enumerate]{leftmargin=.5in}
\setlist[itemize]{leftmargin=.5in}

\usepackage{amsopn}
\graphicspath{{figures/}}

\newcommand {\R}   {\mathbb{R}}

\usepackage{mathtools}
\usepackage{url}

\usepackage{array}

\newlength{\myrowheight}
\newcommand{\WDNote}[1]           
{\textcolor{red}{#1}\marginpar{\textcolor{red}{WD $\longleftarrow$}}}

\begin{document}

\title{Federated Low-Rank Tensor Estimation for Multimodal Image Reconstruction}

\makeatletter
\newcommand{\linebreakand}{%
  \end{@IEEEauthorhalign}
  \hfill\mbox{}\par
  \mbox{}\hfill\begin{@IEEEauthorhalign}
}

\makeatother

\author{\IEEEauthorblockN{Anh Van Nguyen}
\IEEEauthorblockA{
\textit{Northwestern University}\\
Evanston, USA \\
janguyen@u.northwestern.edu}
\and
\IEEEauthorblockN{Diego Klabjan}
\IEEEauthorblockA{
\textit{Northwestern University}\\
Evanston, USA \\
d-klabjan@northwestern.edu}
\and
\IEEEauthorblockN{Minseok Ryu}
\IEEEauthorblockA{
\textit{Arizona State University}\\
Tempe, USA \\
minseok.ryu@asu.edu}

\linebreakand
\IEEEauthorblockN{Kibaek Kim}
\IEEEauthorblockA{
\textit{Argonne National Laboratory}\\
Lemont, USA \\
kimk@anl.gov}
\and
\IEEEauthorblockN{Zichao Di}
\IEEEauthorblockA{
\textit{Argonne National Laboratory}\\
Lemont, USA \\
wendydi@mcs.anl.gov}
}

\maketitle

\begin{abstract}
Low-rank tensor estimation offers a powerful approach to addressing high-dimensional data challenges and can substantially improve solutions to ill-posed inverse problems, such as image reconstruction under noisy or undersampled conditions. Meanwhile, tensor decomposition has gained prominence in federated learning (FL) due to its effectiveness in exploiting latent space structure and its capacity to enhance communication efficiency. In this paper, we present a federated image reconstruction method that applies Tucker decomposition, incorporating joint factorization and randomized sketching to manage large-scale, multimodal data. Our approach avoids reconstructing full-size tensors and supports heterogeneous ranks, allowing clients to select personalized decomposition ranks based on prior knowledge or communication capacity. Numerical results demonstrate that our method achieves superior reconstruction quality and communication compression compared to existing approaches, thereby highlighting its potential for multimodal inverse problems in the FL setting.
\end{abstract}


\section{Introduction}
In recent years, tensors (or high-order arrays) have gained considerable traction in data analysis. A rich body of work demonstrates the effectiveness of tensor decompositions in tasks such as predictive modeling \cite{li2018tucker, ahmed2020tensor, wu2021tensor}, image processing \cite{zhou2013tensor}, feature extraction \cite{hu2017attribute, sorber2015structured} and data fusion \cite{nie2024interpretable, zhou2023multi}, which is the integration of data from different modalities or sources. In particular, the underlying idea of data fusion via tensor decomposition is the joint factorization concept of different datasets into a common latent space. 

Meanwhile, motivated by the substantial overhead and burden of sharing large-scale data across multiple entities (such as privacy concerns), Federated Learning (FL) has emerged as an effective decentralized framework. By allowing clients to collaboratively solve a training problem without transmitting raw data, FL significantly reduces data transfer costs. One of the earliest methods, Federated Averaging (FedAvg) \cite{mcmahan2017communication}, exemplifies this approach by aggregating clients’ locally trained model weights at the server and distributing the resulting global model back to each client for the next training round. In the past years, much efforts have been made to extend FL to scientific applications such as tomographic reconstruction \cite{yang2023hypernetwork, li2023semi, chen2023federated}. A recent approach, called FIRM \cite{byeon2025firmfederatedimagereconstruction}, is proposed to enable federated tomographic
reconstruction from multimodal data by harnessing the complementary relationship
between modalities. Furthermore, an increasing number of studies explore the use of tensor decomposition in FL to leverage the benefits of low-rank representations in a decentralized environment. By factorizing model parameters, clients transmit smaller, structured components rather than full-scale weights, thereby mitigating the communication overhead 
while taking advantage of the low-rank regularization. 

In this work, we expand on the federated tomographic problem in \cite{byeon2025firmfederatedimagereconstruction} by incorporating the Tucker decomposition. Notably, inverse problems, such as tomographic reconstruction,
tend to be ill-posed (i.e., many local minima) due to limited data \cite{di2016optimization,zhou2013tensor}. Low-rank regularization via tensor decomposition helps constrain the solution space by leveraging the underlying low-dimensional structure in the parameter and data space, thereby suppressing noise and artifacts, ultimately improving reconstruction quality \cite{semerci2014tensor, jantre2021low}. Thus, to enforce low-rank regularization, our method employs a projected gradient descent update for local clients, inspired by \cite{ahmed2020tensor}. We further exploit the power of tensor decomposition in addressing multimodal data by utilizing joint factorization and randomized sketching in the server aggregation scheme. 
Our Tucker-based framework is able to significantly improve reconstruction quality and reduce communication overhead. We summarize our contributions as follows.

\begin{itemize} 
    \item We propose a joint factorization approach that streamlines the server aggregation process, eliminating the need to reconstruct full-size tensors while preserving orthogonality of Tucker factor matrices. 
    \item We leverage randomized sketching to reduce the server’s computational burden, facilitating large-scale deployment in federated environments.
    \item Finally, we validate our method with numerical simulations and achieve superior reconstruction quality under noisy and undersampled settings. Moreover, our method outperforms popular compression techniques such as Top-k and sparse encoding. We also show that with heterogeneous ranks, our method is still able to maintain high quality reconstruction.
\end{itemize}

\section{Background and Related Works}
\subsection{Tensor decomposition}\label{subsection:tensor}
Given a tensor $\boldsymbol{\mathcal{X}} \in \mathbb{R}^{n_1 \times n_2 ... \times n_d}$, its matricization (i.e., mode-$k$ unfolding) is the rearrangement of the $k$th dimension into the first dimension of a matrix, and denoted by $\boldsymbol{\mathcal{X}}_{(k)}\in \R^{n_k \times \Pi_{i=1, i \neq k}^d n_i}$. A tensor can be characterized by its rank and the (multilinear) rank of a tensor consists of the column rank of its mode-$k$ unfolding. Tensor-times-matrix (TTM) in mode-$k$ for tensor $\boldsymbol{\mathcal{X}} \in \mathbb{R}^{n_1 \times ... \times n_d}$ with matrix $\mathbf{S} \in \mathbb{R}^{n_k \times r}$ is defined as $\boldsymbol{\mathcal{Y}} 
 = \boldsymbol{\mathcal{X}} \times_k \mathbf{S}$. The result $\boldsymbol{\mathcal{Y}}$ is a tensor of shape $n_1 \times ... \times n_{k-1} \times r \times n_{k+1} \times ... \times n_d$ such that $\mathcal{Y}(i_1, ..., i_{k-1}, j, i_{k+1}, ..., i_d) = \sum_{i_k = 1}^{n_k} \mathcal{X}(i_1, ..., i_k, ..., i_d) \mathbf{S}(i_k, j)$, $j \in [r]$.
Furthermore, given $\mathbf{S}_k \in \mathbb{R}^{n_k \times r_k}$, $k\in[d]$, if $\boldsymbol{\mathcal{Y}} = \boldsymbol{\mathcal{X}} \times_1 \mathbf{S}_1 \times_2 ... \times_d \mathbf{S}_d$, the mode-$k$ unfolding $
\boldsymbol{\mathcal{Y}}_{(k)} = \mathbf{S}_k \boldsymbol{\mathcal{X}}_{(k)}(\mathbf{S}_d \otimes ... \otimes \mathbf{S}_{k+1}\otimes \mathbf{S}_{k-1} \otimes ... \otimes \mathbf{S}_1)^T
$ where $\otimes$ denotes the Kronecker product of two matrices. 

Tensor representation is a powerful tool for handling complex high-dimensional datasets and models. Among the many tensor decompositions, Tucker decomposition is particularly popular because it can be considered an extension of the Singular Value Decomposition (SVD) to higher dimensions \cite{kolda2009tensor}. 
The Tucker rank-$(r_1, ..., r_d)$ decomposition of tensor $\boldsymbol{\mathcal{X}} \in \mathbb{R}^{n_1 \times ... \times n_d}$ is defined as 
$\boldsymbol{\mathcal{X}} \approx [\![\boldsymbol{\mathcal{G}}; \mathbf{S}_1, ..., \mathbf{S}_d]\!]\triangleq \boldsymbol{\mathcal{G}} \times_1 \mathbf{S}_1 \times_2 ... \times_d \mathbf{S}_d$, where $\boldsymbol{\mathcal{G}} \in \mathbb{R}^{r_1 \times ... \times r_d}$ is called the core tensor and  $\mathbf{S}_k \in \mathbb{R}^{n_k \times r_k}$, $k \in [d]$ is a set of factor matrices as the $r_k$ leading left singular vectors of the mode-$k$ unfolding of $\boldsymbol{\mathcal{X}}$. 
Note that this decomposition is not necessarily unique. More detailed discussions of tensors can be found in \cite{kolda2009tensor}. 

Tucker decomposition can be achieved via the higher-order SVD (HOSVD) algorithm \cite{kolda2009tensor}. The sequentially truncated HOSVD (ST-HOSVD) is proposed to improve the computational efficiency of its predecessor \cite{vannieuwenhoven2012new}. Unlike the original algorithm, which treats each dimension independently, ST-HOSVD finds a factor matrix and immediately compresses the tensor in the corresponding dimension. The truncated tensor is then used to find the factor matrix in the next dimension.

Recently, randomized sketching is used to further reduce the cost of SVD. The key idea is to reduce the dimensions of a matrix by generating a randomized sketch, which is smaller but is a good approximation of the original matrix \cite{halko2011finding}. In particular, the Rand-Tucker algorithm \cite{zhou2014decomposition} obtains a sketch 
$\mathbf{Z}$ from the mode-$k$ unfolding of tensor $\boldsymbol{\mathcal{X}}$ and sets the factor matrix $\mathbf{S}_k$  as the orthonormal basis of $\mathbf{Z}$ (via QR decomposition). In this work, we adopt the core idea of Rand-Tucker and modify the randomized procedure to the federated setting.

The framework that underpins most tensor-based data analysis tasks is the generalized low-rank tensor estimation problem \cite{han2022optimal}. Given a dataset $\mathit{D}$, this problem is described as
\begin{equation}\label{eq:estimation}
    \min_{\boldsymbol{\mathcal{X}}} f\big(\boldsymbol{\mathcal{X}}; \mathit{D}\big) \quad
    \text{s.t.} \quad \text{rank}\big(\boldsymbol{\mathcal{X}}_{(k)}\big) \leq r_k, k\in[d]
\end{equation}
where $f\big(\boldsymbol{\mathcal{X}}; \mathit{D}\big)$ is a loss function, and tuple $(r_1, ..., r_d)$ is the expected (multilinear) rank of the tensor of interest. 
In general, this problem can be solved by the projected gradient descent method (e.g., \cite{ahmed2020tensor}) where the gradient descent iterate is projected onto the feasible set via tensor decomposition such as HOSVD to satisfy the multilinear rank constraint. 

\subsection{Tensor decomposition in FL}
Tensor decomposition in the FL framework has previously been studied \cite{yi2023fedlora, qi2024fdlora, park2024communicationefficientfederatedlowrankupdate}. A key aspect in the integration of tensor decomposition in FL is how the server aggregates the factorized components. Mainly, there are two aggregation schemes: the server either reconstructs the full weights before taking the average \cite{dai2023deep, lan2023communication, liu2024marvel}, or adapts FedAVG and directly averages the components \cite{sanchez2024federated, zheng2021distributed, zhang2024federated}. A drawback of the former scheme is the significant increase in computational efforts. The latter combines latent spaces that are independently calculated from separate datasets and can suffer from compounding errors. As noted in \cite{wu2022tenalign}, a potential remedy is joint factorization, whereby these tensors are decomposed into a shared latent factor space. Indeed, joint factorization is a popular approach in data fusion \cite{nie2024interpretable, zhou2023multi} as it provides a systemic framework for merging data obtained from different sources. However, it is still underutilized as an aggregation scheme in FL. In this work, we benchmark against the current aggregation schemes and address their shortcomings using joint factorization. 

Mathematically, given $N$ tensors $\boldsymbol{\mathcal{X}}^i \in \mathbb{R}^{n_1\times...\times n_d}$, let their concatenation be denoted by the tensor $\boldsymbol{\mathcal{X}}$ such that $\boldsymbol{\mathcal{X}} = \big[\boldsymbol{\mathcal{X}}^1\hspace{5pt}|\hspace{5pt} \boldsymbol{\mathcal{X}}^2 \hspace{5pt}|\hspace{5pt} ... \hspace{5pt}|\hspace{5pt} \boldsymbol{\mathcal{X}}^N\big]$ and $\boldsymbol{\mathcal{X}} \in \mathbb{R}^{n_1 \times ... \times n_d \times N}$. The joint factorization of $N$ tensors $\boldsymbol{\mathcal{X}}^i$ is equivalent to finding the factor matrices of the concatenated tensor $\boldsymbol{\mathcal{X}}$ in the first $d$ dimensions. Gao et al. \cite{gao2021federated} propose the first joint factorization-based scheme for the federated feature extraction task. The authors leverage the result in Lemma \ref{lemma:gao} to approximate the left singular vectors of the mode-$k$ unfolding $\boldsymbol{\mathcal{X}}_{(k)}$ as those of $\mathbf{Y} = \big[\mathbf{S}_k^1\boldsymbol{\mathcal{G}}^1_{(k)} \hspace{5pt} ... \hspace{5pt} \mathbf{S}_k^N\boldsymbol{\mathcal{G}}^N_{(k)}\big]$, which effectively bypasses the reconstruction of full-size tensors. 
\begin{lemma}\label{lemma:gao}
(From Gao et al. \cite{gao2021federated}). For $i=1,...,N$ and $k=1,...,d$, let $\boldsymbol{\mathcal{G}}^i \in \mathbb{R}^{r_1\times...\times r_d}$ and $\mathbf{S}^i_k \in \mathbb{R}^{n_k\times r_k}$ such that the columns of $\mathbf{S}^i_k$ are orthogonal. Let $\boldsymbol{\mathcal{X}}^i = [\![\boldsymbol{\mathcal{G}}^i; \mathbf{S}_1^i, ..., \mathbf{S}_d^i]\!]$, $\mathbf{W} = \big[\boldsymbol{\mathcal{X}}^1_{(k)} \hspace{5pt} ... \hspace{5pt} \boldsymbol{\mathcal{X}}^N_{(k)} \big]$ and $\mathbf{Y} = \big[\mathbf{S}_k^1\boldsymbol{\mathcal{G}}^1_{(k)} \hspace{5pt} ... \hspace{5pt} \mathbf{S}_k^N\boldsymbol{\mathcal{G}}^N_{(k)}\big]$. Then $\mathbf{W}$ and $\mathbf{Y}$ have the same set of singular values and if $\mathbf{U}$ and $\mathbf{\bar{U}}$ are the left singular matrix of $\mathbf{W}$ and $\mathbf{Y}$ respectively, then $\mathbf{U} = \bar{\mathbf{U}}\mathbf{P}$ where $\mathbf{P}$ is a unitary block diagonal matrix. 
\end{lemma}
Similarly to \cite{gao2021federated}, we exploit the orthogonality of the factor matrices in the aggregation process and extend it by incorporating randomized sketching to improve computational efficiency. Note that \cite{gao2021federated} focuses on feature extraction while we focus on a multimodal inverse problem, which requires new concepts, ideas, and tricks.

\subsection{Federated tomographic reconstruction}
In this paper, we extend the multimodal image reconstruction problem in \cite{byeon2025firmfederatedimagereconstruction}, which is given by
\begin{equation}\label{eq:firm}
    \min_{\boldsymbol{\mathcal{X}}^1, ..., \boldsymbol{\mathcal{X}}^N} \sum_{i=1}^N f\big(\boldsymbol{\mathcal{X}}^i; \mathit{D}^i\big) \quad
    \text{s.t.} \quad \boldsymbol{\mathcal{X}}_{N} = \sum_j^{N-1} c_j \boldsymbol{\mathcal{X}}^j_{j}
\end{equation}
where $\boldsymbol{\mathcal{X}}^N$ denotes the client with X-ray transmission (XRT) data and $\boldsymbol{\mathcal{X}}^j$, $j\in[N-1]$ are clients with X-ray fluorescence (XRF) data. The coefficient $c_j$ denotes the mass attenuation coefficient of the element of interest and is well characterized.

FIRM, the approach developed in \cite{byeon2025firmfederatedimagereconstruction}, enables clients with data obtained from different modalities to collaborate and improve their respective solution quality with convergence guarantee \cite{byeon2025firmfederatedimagereconstruction}. Specifically, at epoch $t$, clients use gradient descent to locally update the previous solution $\boldsymbol{\mathcal{X}}^i(t-1)$ and obtain $\tilde{\boldsymbol{\mathcal{X}}}^i$, $i\in[N]$. After receiving the full-size tensors $\tilde{\boldsymbol{\mathcal{X}}}^i$, the server performs a series of arithmetic operations to enforce the multimodality constraint as follows.

\begin{equation}\label{eq:firm_update}
\begin{aligned}
    \Sigma &\leftarrow \sum_{i=1}^{N-1} c_i\tilde{\boldsymbol{\mathcal{X}}}^i\\
    \boldsymbol{\mathcal{X}}^i(t) &\leftarrow \tilde{\boldsymbol{\mathcal{X}}}^i + \dfrac{c_i}{2}(\tilde{\boldsymbol{\mathcal{X}}}^N - \Sigma), i\in[N-1]\\ 
    \boldsymbol{\mathcal{X}}^N(t) &\leftarrow \dfrac{\tilde{\boldsymbol{\mathcal{X}}}^N + \Sigma}{2}\\
\end{aligned}
\end{equation}
Previous works have shown that the low-rank regularization via tensor decomposition improves standard tomographic reconstruction quality \cite{jantre2021low}. In this paper, we introduce low-rank regularization and incorporate tensor decomposition to the FIRM framework to further improve the solution of Problem~\ref{eq:firm}. 

\section{Methodology}
In this section, we build on the projected gradient descent approach from \cite{ahmed2020tensor} and develop a joint factorization scheme for the server; we also adapt the FIRM update \eqref{eq:firm_update} to enforce the multimodality constraint in the Tucker component space. 

\subsection{Algorithm} We consider a constrained low-rank tensor estimation problem in the federated setting as follows, 
\begin{equation}\label{eq:fedTucker}
\begin{aligned}
    \min_{\boldsymbol{\mathcal{X}}^1, ...,\boldsymbol{\mathcal{X}}^N} & \sum_i^N f(\boldsymbol{\mathcal{X}}^i; \mathit{D}^i\big)\\
    \text{s.t.} \quad & \text{rank}\big(\boldsymbol{\mathcal{X}}^i_{(k)}\big) \leq r_k \quad k\in[d], i\in[N]\\
    & \boldsymbol{\mathcal{X}}^i =[\![\boldsymbol{\mathcal{G}}^i; \mathbf{S}_1, ..., \mathbf{S}_d]\!] \quad i \in [N] \\
     & \mathbf{S}_k^T \mathbf{S}_k = \mathbf{I}_{r_k} \quad k \in [d] \\
     & \boldsymbol{\mathcal{G}}^{N} = \sum_j^{N-1} c_j \boldsymbol{\mathcal{G}}^j \\
\end{aligned}
\end{equation}
where the core tensors $\boldsymbol{\mathcal{G}}^i \in \mathbb{R}^{r_1 \times...\times r_d}$ $\forall i\in [N]$ are client-specific, the factor matrices $\mathbf{S}_k \in \mathbb{R}^{n_k \times r_k}$ $\forall k \in [d]$ are shared across local clients, and $\mathbf{I}_{r_k} \in \mathbb{R}^{r_k \times r_k}$ is the identity matrix. This structure ensures that while clients retain their data locally, global knowledge can still be learned collaboratively. Note that the linear constraint in Problem~\ref{eq:firm} naturally applies to $\boldsymbol{\mathcal{G}}^i$. The algorithm is provided in Algorithm \ref{alg:full}.

At local clients, we use the projected gradient descent method to optimize the low-rank tensor estimation problem \eqref{eq:estimation}, where ST-HOSVD is used for projection due to its robustness comparing to HOSVD. With this local update, clients obtain their independent set of factor matrices and the corresponding cores. For the global stage on the server, we develop an aggregation scheme based on joint factorization to combine the latent spaces across clients while maintaining the orthogonality condition $\mathbf{S}_k^T \mathbf{S}_k = \mathbf{I}_{r_k}$, $k \in [d]$. We also utilize randomized sketching in the factorization procedure for better computational efficiency. To enforce the last constraint in \eqref{eq:fedTucker}, we leverage the FIRM update to aggregate the cores. Our aggregation scheme avoids reconstructing full-size tensors and facilitates heterogeneous ranks, compared to other works \cite{dai2023deep, liu2024marvel, zheng2021distributed, sanchez2024federated, zhang2024federated}. Since both the server and the clients send low-rank components, our method promotes compression for both upstream and downstream communication. 

\RestyleAlgo{ruled}
\SetKwComment{Comment}{/* }{ */}
\SetKwInOut{Server}{Server Input}
\SetKwInOut{Client}{Client Input}
\begin{algorithm}
\caption{Federated Low-rank Tensor Estimation with Joint Factorization}\label{alg:full}
\setcounter{AlgoLine}{0}
\SetAlgoLined
\Server{Tucker rank $(r_1, ..., r_d)$, number of epochs $T$, learning rate $\eta$}
\Client{Dataset $\mathit{D}^i$, $i\in[N]$}
\textbf{Server} initializes and broadcasts to clients $\boldsymbol{\mathcal{G}}^i(0)$, $ i \in [N]$ and $\mathbf{S}_k(0)$, $k \in [d]$;

\For{$t = 1, ..., T$}{
\textbf{CLIENT $i=1,...,N$:}

$\boldsymbol{\mathcal{X}}^i(t) \leftarrow \boldsymbol{\mathcal{G}}^i(t-1) \times_1 \mathbf{S}_1(t-1) \times_2 ... \times_d \mathbf{S}_d(t-1) $;

$\tilde{\boldsymbol{\mathcal{X}}}^i \leftarrow \boldsymbol{\mathcal{X}}^i(t) - \eta \nabla f(\boldsymbol{\mathcal{X}}^i(t); \mathit{D}^i)$;

$\hat{\boldsymbol{\mathcal{G}}}^i, \hat{\mathbf{S}}^i_1, ..., \hat{\mathbf{S}}^i_d \leftarrow \text{ST-HOSVD} (\tilde{\boldsymbol{\mathcal{X}}}^i)$;

Send $\hat{\boldsymbol{\mathcal{G}}}^i, \hat{\mathbf{S}}^i_1, ..., \hat{\mathbf{S}}^i_d$ to the Server;

\textbf{SERVER:}

$\mathbf{S}_1(t), ..., \mathbf{S}_d(t) \leftarrow$ Joint-Factorization($\hat{\boldsymbol{\mathcal{G}}}^i, \hat{\mathbf{S}}_1^i, ..., \hat{\mathbf{S}}_d^i$, $i\in[N]$);

Broadcasts $\mathbf{S}_k(t)$, $k \in [d]$ to all clients;

\For{$i = 1, ..., N$}{
    $\tilde{\boldsymbol{\mathcal{G}}}^i \leftarrow [\![\hat{\boldsymbol{\mathcal{G}}}^i; {\mathbf{S}_1(t)}^T\hat{\mathbf{S}}^i_1, ..., {\mathbf{S}_d(t)}^T\hat{\mathbf{S}}^i_d]\!]$;
    }
    
$\Sigma \leftarrow \sum_{i=1}^{N-1} c_i\tilde{\boldsymbol{\mathcal{G}}}^i$;

$\boldsymbol{\mathcal{G}}^i(t) \leftarrow \tilde{\boldsymbol{\mathcal{G}}}^i + \dfrac{c_i}{2}(\tilde{\boldsymbol{\mathcal{G}}}^N - \Sigma)$, $i\in[N-1]$;
    
$\boldsymbol{\mathcal{G}}^N(t) \leftarrow \dfrac{\tilde{\boldsymbol{\mathcal{G}}}^N + \Sigma}{2}$;

Sends $\boldsymbol{\mathcal{G}}^i(t)$ to client $i$;
}
\end{algorithm}
Let $\boldsymbol{\mathcal{G}}^i(t)$, $i \in [N]$ and $\mathbf{S}_k(t)$, $k \in [d]$ denote the cores and factor matrices obtained at the end of epoch $t$. In lines 4-6 of Algorithm \ref{alg:full}, a client $i$ constructs the full tensor $\boldsymbol{\mathcal{X}}^i(t)$ from the core $\boldsymbol{\mathcal{G}}^i(t-1)$ and factor matrices $\mathbf{S}_k(t-1)$, $k\in [d]$ received from the server. The client solves the local low-rank optimization problem by updating its respective tensor $\boldsymbol{\mathcal{X}}^i(t)$ using gradient descent to obtain $\tilde{\boldsymbol{\mathcal{X}}}^i$ and applying ST-HOSVD to this new tensor. The client then sends the factorized components $\hat{\boldsymbol{\mathcal{G}}}^i, \hat{\mathbf{S}}^i_1, ..., \hat{\mathbf{S}}^i_d$ to the server. 

To enforce the first two constraints in~\eqref{eq:fedTucker}, the server performs joint factorization. One option is to directly apply Lemma \ref{lemma:gao} from \cite{gao2021federated} so that the server can perform joint factorization
without full-size tensors (as summarized in Algorithm~\ref{alg:server_hosvd}). 
\begin{algorithm}[h]
\caption{Joint factorization on the server at epoch $t$ (Direct application of \cite{gao2021federated})}\label{alg:server_hosvd}
\setcounter{AlgoLine}{0}
\SetAlgoLined
\KwIn{Tucker rank $(r_1, ..., r_d)$}
\KwIn{$\hat{\boldsymbol{\mathcal{G}}}^i, \hat{\mathbf{S}}_1^i, ..., \hat{\mathbf{S}}_d^i$, $i\in[N]$, $t$}
\KwOut{$\mathbf{S}_k(t)$, $k\in[d]$}
\For{$k = 1, ..., d$}{
    $\mathbf{Y}_k \leftarrow$  $\Big[\hat{\mathbf{S}}^1_k \hat{\boldsymbol{\mathcal{G}}}^1_{(k)} \quad ... \quad \hat{\mathbf{S}}^N_k \hat{\boldsymbol{\mathcal{G}}}^N_{(k)}\Big]$;
    
    $\mathbf{S}_k(t) \leftarrow r_k$ left singular vectors of $\mathbf{Y}_k$; 
}
\end{algorithm}
However, it involves the SVD calculation of large matrices $\mathbf{Y}_k$ of size $n_k \times N\Pi_{j \neq k}^d r_j$, which can be very expensive as the number of clients and dimensions grow. 

Therefore, to further improve efficiency, we consider QR decomposition-based randomized Tucker decomposition (as proposed in~\cite{zhou2014decomposition}) and adapt it to the federated setting as follows. 
\begin{lemma}\label{lemma:communication}
     Suppose that $\boldsymbol{\mathcal{X}}^i = [\![\boldsymbol{\mathcal{G}}^i; \mathbf{S}_1^i, ..., \mathbf{S}_d^i]\!]$ where ${\mathbf{S}_k^i}^T \mathbf{S}_k^i = \mathbf{I}_{r_k}$ $\forall i \in [N], \forall k \in [d]$. Let $\mathbf{W}_k = \big[\boldsymbol{\mathcal{X}}^1_{(k)} \hspace{5pt} ... \hspace{5pt} \boldsymbol{\mathcal{X}}^N_{(k)} \big]$. If $\mathbf{Z}_k=\mathbf{W}_{k}\mathbf{\Omega}_k$ for a Gaussian matrix $\mathbf{\Omega}_k$ of size $N\Pi_{j \neq k}^d n_j \times r_k$, then $\mathbf{Z}_k= \sum_{i=1}^N \mathbf{S}_k^i\boldsymbol{\mathcal{G}}^i_{(k)}\mathbf{\Omega}_{k}^i$ for some $\mathbf{\Omega}_{k}^i$ Gaussian matrices of size $\Pi_{j \neq k}^d r_j \times r_k$.    
\end{lemma}
\begin{proof}
    We rewrite $\mathbf{\Omega}_k = \begin{bmatrix}
           \bar{\mathbf{\Omega}}_{k}^1 \\
           \vdots \\
           \bar{\mathbf{\Omega}}_{k}^N
         \end{bmatrix}$ where $\bar{\mathbf{\Omega}}_{k}^i \in \mathbb{R}^{\Pi_{j \neq k}^d n_j \times r_k}$. We have $\mathbf{W}_{k}\mathbf{\Omega}_k = \sum_{i=1}^N \mathbf{W}^i_{(k)}\bar{\mathbf{\Omega}}_{k}^i$. In addition, since $\boldsymbol{\mathcal{X}}^i = [\![\boldsymbol{\mathcal{G}}^i; \mathbf{S}_1^i, ..., \mathbf{S}_d^i]\!]$, the mode-$k$ unfolding is $\mathbf{W}^i_{(k)} = \mathbf{S}_k^i\boldsymbol{\mathcal{G}}^i_{(k)}{\mathbf{V}_k^i}^T$ where $\mathbf{V}_k^i = (\mathbf{S}^i_d \otimes ... \otimes \mathbf{S}^i_{k+1}\otimes \mathbf{S}^i_{k-1} \otimes ... \otimes \mathbf{S}^i_1)$. Because ${\mathbf{S}_k^i}^T \mathbf{S}_k^i = \mathbf{I}_{r_k}$, it is easy to show that $\mathbf{V}_k^i$ also has orthogonal columns. 

    We now have $\mathbf{W}_{k}\mathbf{\Omega}_k = \sum_{i=1}^N \mathbf{S}_k^i\boldsymbol{\mathcal{G}}^i_{(k)}{\mathbf{V}_k^i}^T\bar{\mathbf{\Omega}}_{k}^i$. Let $\mathbf{\Omega}_{k}^i = {\mathbf{V}_k^i}^T\bar{\mathbf{\Omega}}_{k}^i$. Using the property of the Gaussian distribution, it is readily observed that $\mathbf{\Omega}_{k}^i$ is also a Gaussian matrix with each element independently drawn from the normal distribution. 
\end{proof}

Finally, we propose Algorithm \ref{alg:server_rhosvd} that combines the result in Lemma \ref{lemma:communication} (shown in lines 2-3) and QR decomposition. Compared to Algorithm \ref{alg:server_hosvd}, we now decompose a smaller matrix of size $n_k \times r_k$, which is independent of the number of clients and dimensions. Note that with randomized factorization, we essentially sample from the column space of a matrix, so as $r_k$ decreases, the randomized sketch becomes a less accurate approximation of the original tensor \cite{halko2011finding}. This in turn can introduce more errors in subsequent QR decomposition. Therefore, if we choose small values for the Tucker ranks, the loss in accuracy due to randomization may outweigh the gain in computational efficiency. In this case, it is advisable to use Algorithm \ref{alg:server_hosvd} instead.     
We name our method Component Randomized Joint Factorization (i.e., CompRandJF) with the use of Algorithm \ref{alg:server_rhosvd} for joint factorization in Algorithm \ref{alg:full}. Alternatively, we term the method using Algorithm \ref{alg:server_hosvd} for joint factorization in Algorithm \ref{alg:full} as Component Joint Factorization (i.e., CompJF).  

\begin{algorithm}
\caption{Randomized joint factorization on the server at epoch $t$}\label{alg:server_rhosvd}
\setcounter{AlgoLine}{0}
\SetAlgoLined
\KwIn{Tucker rank $(r_1, ..., r_d)$}
\KwIn{$\hat{\boldsymbol{\mathcal{G}}}^i, \hat{\mathbf{S}}_1^i, ..., \hat{\mathbf{S}}_d^i$, $i\in[N]$, $t$}
\KwOut{$\mathbf{S}_k(t)$, $k\in[d]$}
\For{$k = 1, ..., d$}{
    $\mathbf{\Omega}_{k}^i \leftarrow$ $i$'th i.i.d. sample of a Gaussian matrix of size $\Pi_{j \neq k}^d r_j \times r_k$, $i\in[N]$;

    $\mathbf{Y}_k \leftarrow \sum_{i=1}^N \hat{\mathbf{S}}^i_k \hat{\boldsymbol{\mathcal{G}}}^i_{(k)} \mathbf{\Omega}_{k}^i$;
    
    $\mathbf{S}_k(t) \leftarrow \mathbf{Q}$ from QR decomposition of $\mathbf{Y}_k$;
}
\end{algorithm}
To satisfy the last constraint in \eqref{eq:fedTucker}, the server first recomputes the cores after obtaining the updated factor matrices (see lines 9–11 of Algorithm \ref{alg:full}), then applies the FIRM update \eqref{eq:firm_update}, adapted specifically for the cores.
\begin{remark}
    Heterogeneous ranks can arise when clients have varying bandwidths or have prior knowledge about the appropriate rank of the local data. Our joint factorization scheme for the server allows heterogeneous ranks, as it does not require dimensions to match and clients can choose their own Tucker rank $(r_1^i, ..., r_d^i)$. Essentially, Algorithms \ref{alg:server_hosvd} and \ref{alg:server_rhosvd} can be adjusted to find the $r_k^*$ leading left singular vectors where $r_k^* = \text{max}(r_k^1, ..., r_k^N)$, $k\in[d]$.
\end{remark}

\subsection{Complexity Analysis}\label{method:complexity}
In this section, we discuss the complexity regarding the communication of our proposed method. Without loss of generality, we assume that $n = n_1 = ... = n_d$ and $r = r_1 = ... = r_d$. For a tensor $\boldsymbol{\mathcal{X}} \in \mathbb{R}^{n_1 \times ... \times n_d}$, the communication complexity of transmitting the full-size tensor is $ O (n^d) $. On the other hand, communicating the rank-$(r, ..., r)$ Tucker decomposition reduces the complexity to $O(r^d + dnr)$. 

A key challenge in utilizing tensor decomposition for communication compression is the choice of rank $r$ that maximizes the compression ratio (defined as $\phi = \dfrac{n^d}{r^d + dnr}$) while maintaining performance. Dai et al. \cite{dai2023deep} give an upper bound for the Tucker rank that guarantees a reduction in communication cost per epoch, specifically $r \leq \dfrac{n}{(1 + dn)^{1/d}}$. While this bound provides practical guidance, it is overly restrictive in the range of $r$. For example, when $n=250$ and $d=2$, the bound suggests $r \leq 11$. Such aggressive compression, and consequently regularization on the underlying structure of the solution, can have a negative impact on the reconstruction quality. 

To address this limitation, it is important to find a broader range of values for $r$. Proposition \ref{communication} provides less restrictive bounds for $d=2$ and $d=3$ that allow clients to choose a much larger Tucker rank, compared to \cite{dai2023deep}. For the same example, the acceptable range of $r$ increases to $r \leq 103$. Proposition \ref{communication} further shows that as $d$ and $n$ increase, the communication overhead savings with Tucker decomposition increase substantially. 

\begin{proposition}\label{communication}
If $n\geq 3$, then the improved upper bound reads
\begin{itemize}
    \item $r < \dfrac{n}{\sqrt{2}+1}$ for $d=2$, and
    \item $r < n\big(\dfrac{n-3}{n}\big)^{1/3}$ for $d=3$.
\end{itemize}
Furthermore, for both values of $d$, these bounds are greater than $\dfrac{n}{(1+dn)^{1/d}}$.
\end{proposition}
\begin{proof}
Let $\beta = \dfrac{r}{n}$, resulting in $\phi = \dfrac{1}{\beta^d + (d\beta)/n^{d-2}}$.
Note that $0 < \beta < 1$ since $r < n$.

For $d=2$, $\phi = \dfrac{1}{\beta^2 + 2\beta}$. If $\beta < 1/(\sqrt{2} + 1)$, then $\beta^2 + 2\beta < 1$ and $\phi > 1$. Since $(1 + dn)^{1/d} \geq \sqrt{5} > \sqrt{2} + 1$, the bound $\dfrac{n}{\sqrt{2}+1} > \dfrac{n}{(1+2n)^{1/2}}$.

For $d=3$, $\phi = \dfrac{1}{\beta^3 + 3\beta/n}$. If $\beta < \big(\dfrac{n-3}{n}\big)^{1/3}$, we have $\beta^3 + 3\dfrac{\beta}{n} < \dfrac{n-3}{n} + \dfrac{3}{n} = 1$ and $\phi > 1$. Since $1 + 3n > 1 + \dfrac{3}{n-3} = \dfrac{n}{n - 3}$, the bound $n\big(\dfrac{n-3}{n}\big)^{1/3} > \dfrac{n}{(1+3n)^{1/3}}$.    
\end{proof}

We also compare the computational complexity of CompJF and CompRandJF with the baseline approach, where the server reconstructs full-size tensors to aggregate before applying tensor decomposition (termed as FullDecomp). We summarize the operation counts in Table \ref{table:complexity}.

\renewcommand{\arraystretch}{1.25}
\begin{table}[h] 
\caption{Computational complexity of the server per epoch}
\label{table:complexity}
\centering
\begin{tabular}{|c | c |}
\hline
\hline
Approach & Complexity \\
\hline
FullDecomp & $\makecell{O\big(N\sum_{k=1}^d \big( n^{d-k+1}r^k + \\ n^{d-k+2} r^{k-1} + n^{d-k} r^k) + Nn^d\big) 
}$  \\
\hline
CompJF (Alg. \ref{alg:full} \& \ref{alg:server_hosvd}) & $\makecell{O\big(Nd\big(nr^d + n^2r^{d-1}\big) \\+ N\big(dnr^2 + dr^{d+1}\big) + Nr^d\big)}$ \\
\hline
CompRandJF (Alg. \ref{alg:full} \& \ref{alg:server_rhosvd}) & \makecell{$O\big(Nd\big(r^d + r^{d+1} + nr^2 + nr\big)+ dnr^2$ \\$+ N\big(dnr^2 + dr^{d+1}\big) + Nr^d\big)$} \\
\hline
\hline
\end{tabular}
\end{table}

\section{Experiments and Results}
We evaluate the performance of CompJF and CompRandJF using synthetic data. To show the benefits of our aggregation schemes, we benchmark them against FIRM and FullDecomp, where the server reconstructs full-size tensors to aggregate before applying tensor decomposition for downstream communication. We also experiment with the naive extension of FedAVG in which factor matrices are averaged across clients but this aggregation scheme does not yield meaningful training. 

In addition, we demonstrate the communication efficiency achieved by CompJF and CompRandJF. For benchmarking, we apply to FIRM the compressed sparse row (CSR) encoding and Top-$k$ sparsification. Top-$k$ is a magnitude-based sparsification method and has been shown to effectively reduce communication overhead \cite{sattler2019robust, malekijoo2021fedzip}. For Top-$k$ sparsification, we first sort all the values in an array and only keep the the top $k\%$ of the values; the remaining values are set to zero. Note that Top-$k$ sparsification produces sparse arrays and can further benefit from sparse encoding. Thus, the two benchmarks for communication efficiency are FIRM with only CSR encoding and FIRM with a combination of Top-$k$ and CSR encoding. We refer to the second benchmark as Top-$k$ for short.

Finally, we highlight the ability of our methods to facilitate heterogeneous ranks in which clients determine their own Tucker decomposition ranks.

\subsection{Experimental settings}
Tomographic data, or sinograms, are generated from scanning an object using an energy source with $\theta$ angles and $\tau$ discretized beamlets for each angle. Let the tensor $\boldsymbol{\mathcal{A}} \in \mathbb{R}^{\theta \times \tau \times n_1 \times n_2}$ represent the discrete Radon transform and $\boldsymbol{\mathcal{A}}_{\theta, \tau, i, j}$ is the intersection length of the angle-beamlet pair $(\alpha, \beta)$ with the pixel $(i, j)$. Let the tensor to be reconstructed and the observed measurements of each client $i$ be denoted by $\boldsymbol{\mathcal{X}}^i \in \mathbb{R}^{n_1 \times ... \times n_d}$ and $\boldsymbol{\mathcal{B}}^i \in \mathbb{R}^{\theta \times \tau \times n_3 ... \times n_d}$. The loss function in our tensor estimation problem becomes $f =  \sum_{i=1}^N \big|\big|\boldsymbol{\mathcal{A}} \boxtimes \boldsymbol{\mathcal{X}}^i - \boldsymbol{\mathcal{B}}^i \big|\big|^2_F$ where $N$ is the total number of local clients and $\boxtimes$ denotes the \textit{contracted tensor product} of two tensors.

We evaluate the proposed method on synthetic data, which are generated with parallel geometry and are variants of the Shepp-Logan phantom, a standard image for testing in computed tomography \cite{gach20082d}. Following the setting in \cite{byeon2025firmfederatedimagereconstruction}, we experiment with $N=4$ clients, one holding data obtained from the $XRT$ modality and the other with data produced by the $XRF$ modality. The ground truth $\boldsymbol{\mathcal{X}}^i$ to be reconstructed is of size $250 \times 250$, and the Radon transform has $\tau=354$ beamlets and $\theta=100$ angles. Note that with this number of angles, the measurements are undersampled and the inverse problem is ill-posed.

\begin{figure*}
\centering
  \includegraphics[width=0.9\linewidth]{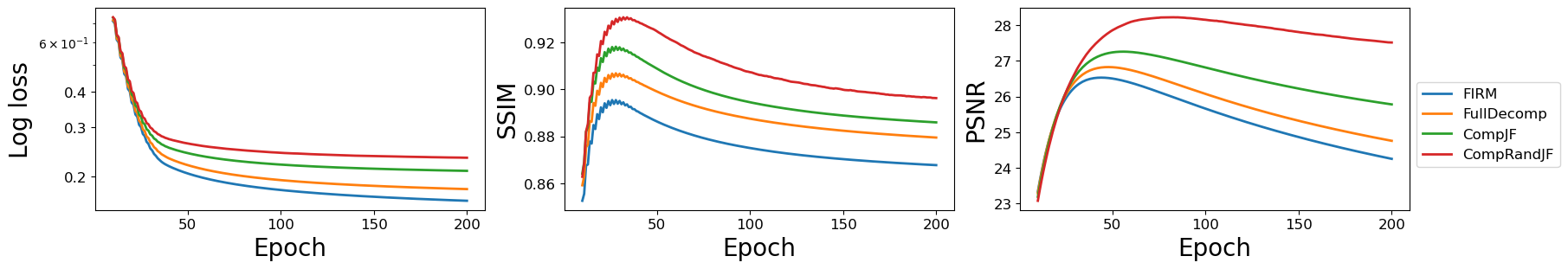}
  \caption{\footnotesize{Evaluating the performance of various methods given the strongest noise level ($\sigma=0.1$). For CompJF, CompRandJF and FullDecomp, we present the results for Tucker rank $r=100$, the largest value that guarantee communication compression. The results for other ranks $r \geq 40$ is similar.}}
\label{figure:noisy_data}
\end{figure*}

We experiment with adding speckle noise, which is inherent in the imaging process, to the sinograms, such that $\textit{output} = \textit{true signal} + \epsilon*\textit{true signal}$ and $\epsilon$ is normally distributed. We vary the standard deviation of the added noise as $\sigma \in \{0.01, 0.05, 0.1\}$.

For CompJF, CompRandJF and FullDecomp, we use Proposition \ref{communication} to obtain the upper bound of the decomposition rank $r$ given that $n=250$. The results are obtained with $r \in \{10, 20, 40, 60, 80, 100\}$. In the heterogeneous rank setting, each client randomly samples a value of $r$ in the range of $[20, 100]$. There are two possible scenarios of heterogeneous ranks. Clients a) have prior knowledge of the appropriate rank or fixed communication limit, or b) have unstable communication bandwidths, so the Tucker ranks need adjusting each communication round. Consequently, in our experiments, the Tucker ranks can be a) sampled before the federated reconstruction starts or b) resampled every epoch prior to upstream communication.
 
\subsection{Performance metrics}
\subsubsection{Reconstruction quality} We compare the solutions to the ground-truth images by evaluating two metrics as the Peak Signal-to-Noise Ratio (PSNR) and the multiscale Structural Similarity Index Measure (SSIM). PSNR is inversely proportional to the mean squared error (MSE) and SSIM computes similarities in structural information and incorporates interdependencies among pixels. SSIM is bounded between $-1$ and $1$ and when $SSIM=0$, there is no similarity between two images. For both quantities, higher values indicate better reconstruction.

\subsubsection{Communication efficiency} 
We are interested in measuring the trade-off between reconstruction quality and communication cost. Song et al. \cite{song2023federated} propose Gamma Communication Efficiency $    GCE = \dfrac{\text{Accuracy}}{(1 - \text{Accuracy})^\gamma \sum_{t=1}^T log_2(V_t + 1)}
$ in which $T$ denotes the number of communication rounds and $V_t$ is the communication volume, measured in bits. The parameter $\gamma$ controls the importance of test accuracy in relation to communication, and the two quantities are near proportional if $\gamma \rightarrow 0$. A higher value of GCE indicates a better balance between communication and accuracy, which is more desirable. We consider image quality and communication compression to be of equal importance and use $\gamma=0.01$, as in \cite{song2023federated}, to analyze our results. We also adapt the metrics to images by replacing the test accuracy with SSIM as the SSIM values in our experiments fall in the range $[0, 1]$. To calculate $V_t$, we combine upstream and downstream communication. 

\subsection{Results}\label{subsection:results}

\subsubsection{\textbf{Robustness against noisy data}} We demonstrate in Fig. \ref{figure:noisy_data} that CompJF and CompRandJF produce better reconstruction than other approaches in the presence of noise. To understand the effect of noise, we plot the results for the strongest noise level ($\sigma=0.1$) and use the largest Tucker rank that guarantee communication compression ($r=100$). 

Overfitting to noise, combined with the ill-posed nature of the inverse problem, often leads to a degradation of image quality as the optimization progresses. This is evident in Fig. \ref{figure:noisy_data} as both PSNR and SSIM peak before 100 epochs and then decrease over time, while the loss remains non-increasing. However, the low-rank decomposition coupled with joint factorization (as enforced in CompJF and CompRandJF) can mitigate this effect, and result in a more graceful degradation. This behavior aligns with results from previous works in which low-rank regularization helps constrain the solution space of the inverse problem, suppressing noise and artifacts in the subsequent reconstructed images. This also explains why the other three methods lead to better quality metrics compared to FIRM despite having worse loss function.

We also show the ground truth and the reconstructed images from various methods in Fig. \ref{appendix_fig:recon_img}. The quality of the images aligns with the numerical results in Fig. \ref{figure:noisy_data}. It is visible that CompJF and CompRandJF have additional denoising effect on the reconstructed images compared to FIRM.
\begin{figure}[htb]
\centering
  \includegraphics[width=0.9\linewidth]{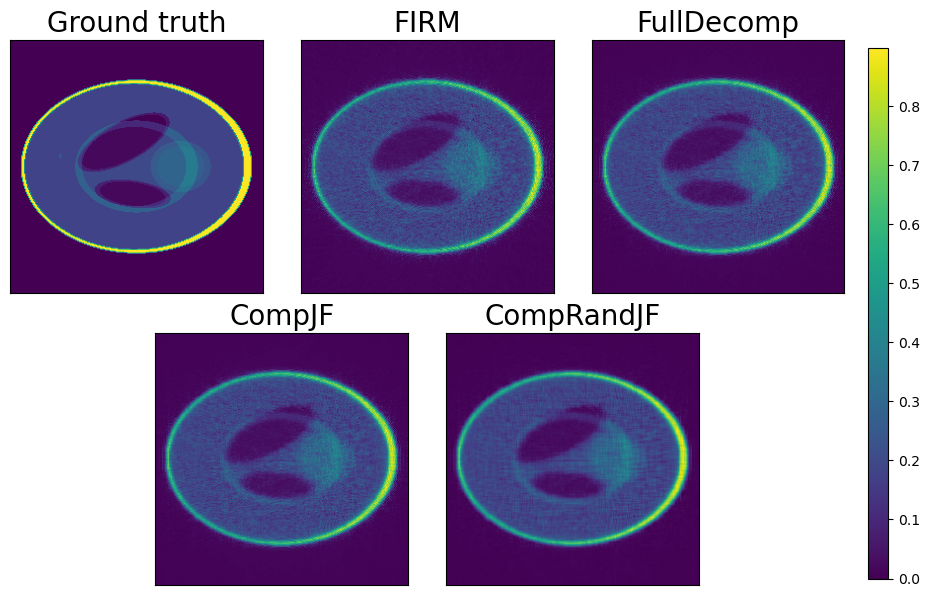}
  \caption{\footnotesize{Ground-truth and reconstructed images from various methods for client $N$ with XRT imaging modality. The images are reconstructed from noisy data ($\sigma=0.1$). For FullDecomp, CompJF, and CompRandJF, we use Tucker rank $r=100$. The reconstructed images correspond to the highest SSIM values.}}
\label{appendix_fig:recon_img}
\end{figure}
\subsubsection{\textbf{Stable performance across Tucker ranks}} In addition to better reconstruction, CompJF and CompRandJF are able to maintain good performance as we decrease the Tucker rank and further compress the data. We present the best result for each rank value as well as the result induced by early stopping in Fig. \ref{figure:hyperparameters}. A common solution to overfitting in tomographic reconstruction is to apply the discrepancy principle, which is an early stopping rule and is widely used in the literature. We use the formulation of the discrepancy principle as in \cite{byeon2025firmfederatedimagereconstruction} so that the optimization stops when the condition $
    f^i(t) \leq \max{\big(\boldsymbol{\mathcal{B}}^i\big)}\sqrt{\theta\tau} \sigma
$ is satisfied for every client, where $\boldsymbol{\mathcal{B}}^i$ are the observed measurements, $\theta$ is the number of projection angles, $\tau$ is the number of beamlets and $\sigma$ is the standard deviation of added noise.

We observe that our proposed approaches produce better results than FIRM for Tucker ranks $r \geq 40$, especially as the data become more noisy. This performance holds when early stopping is incorporated. For $r < 40$, CompJF exhibits a more gradual decrease in performance compared to CompRandJF. This is due to the randomized sketch used in CompRandJF. As $r$ further decreases, this sketch is no longer a good approximation of the original matrix and leads to loss of information in the factorization process.

\begin{figure}[h]
\centering
\begin{subfigure}{1\linewidth}
    \includegraphics[width=\linewidth]{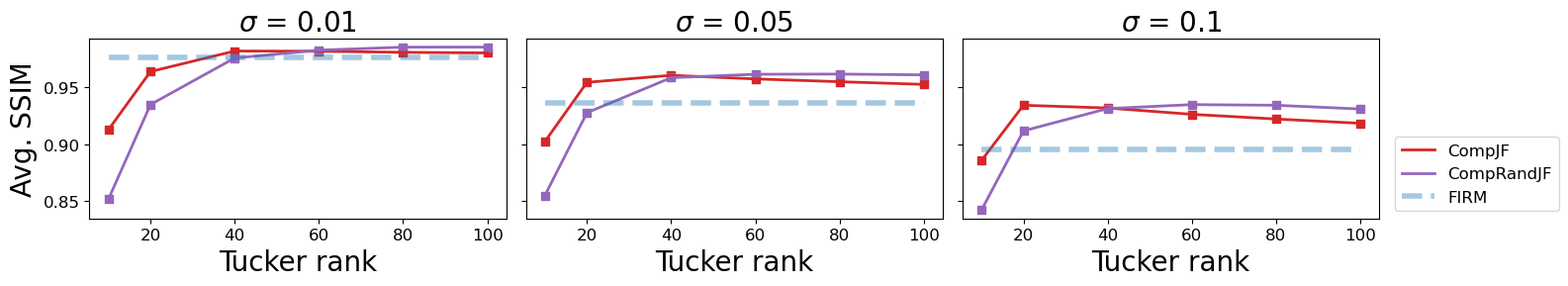}
    \caption{\footnotesize{Best result overall}}
\end{subfigure}
\\
\begin{subfigure}{1\linewidth}    \includegraphics[width=\linewidth]{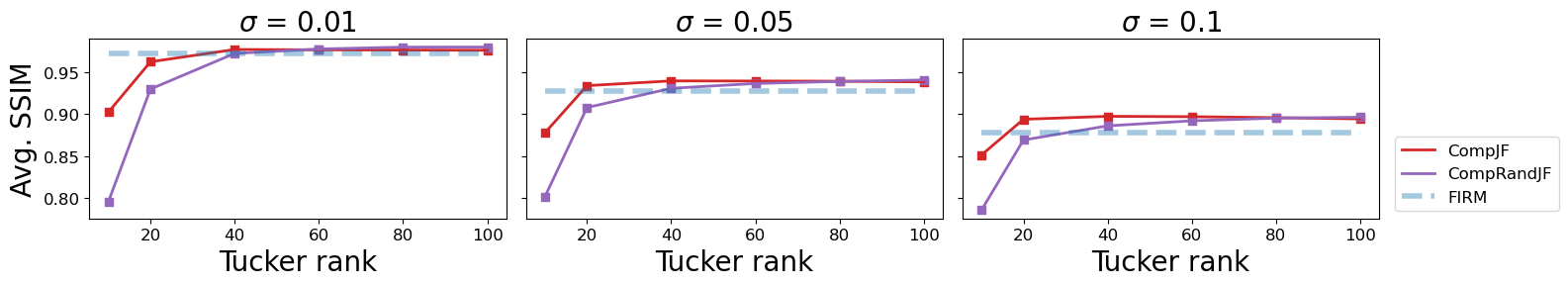}
    \caption{\footnotesize{Result with early stop}}
\end{subfigure}
  \caption{\footnotesize{Performance of CompJF and CompRandJF for varying Tucker decomposition ranks and varying level of noise. (a) Includes the highest SSIM value, averaged across clients, obtained by the approaches. (b) Includes the average SSIM obtained when the discrepancy principle is satisfied. With early stopping condition, all three approaches terminate at the same epoch.}}
\label{figure:hyperparameters}
\end{figure}

\subsubsection{\textbf{Communication compression}} We next illustrate the superior ability to balance reconstruction quality and communication compression of our methods, compared to Top-$k$ sparsification and CSR encoding. With Top-$k$ sparsification, the parameter $k$ controls the percentage of compression in each communication round, and we vary $k \in \{10, 30, 50, 70, 90\}$. For easier comparison, we convert the chosen Tucker ranks to the corresponding percentage of compression. For example, given that $n=250$, if $r=40$, we have a $34.56\%$ compression. 
\begin{figure}[h]
\centering
\begin{subfigure}{1\linewidth}
    \includegraphics[width=\linewidth]{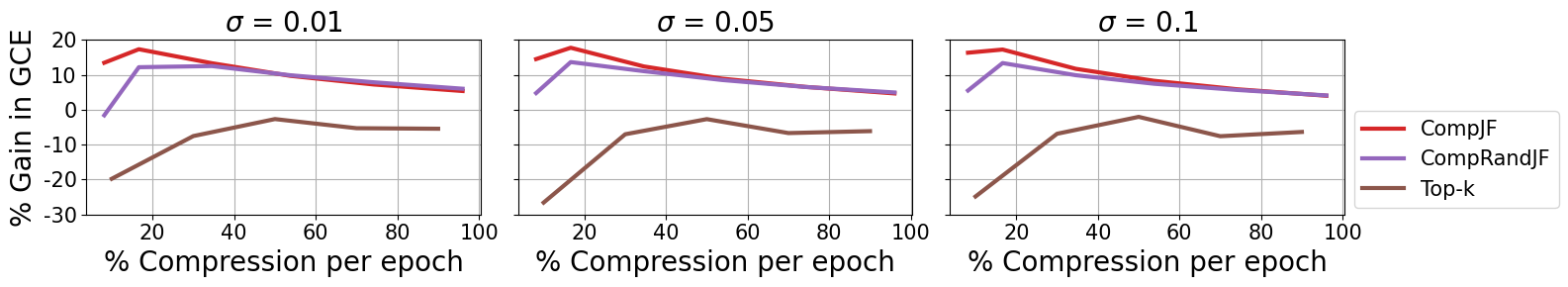}
    \caption{\footnotesize{\% Gain in GCE by CompJF, CompRandJF and FIRM with Top-$k$ over FIRM with CSR encoding.}}
\end{subfigure}
\\
\begin{subfigure}{1\linewidth}
    \includegraphics[width=\linewidth]{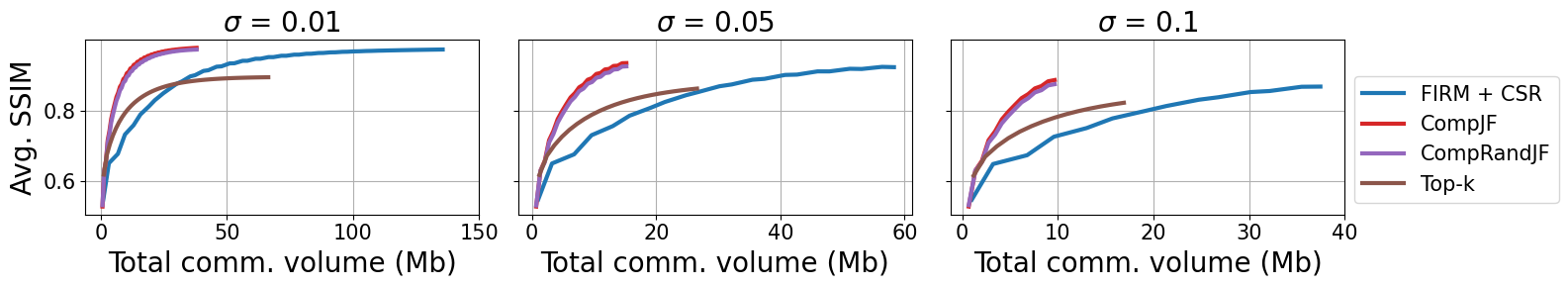}
    \caption{\footnotesize{Plot of average SSIM and the total communication volume in Mb to achieve such SSIM.}}
\end{subfigure}
  \caption{\footnotesize{Performance of our methods, CompJF and CompRandJF, Top-k and CSR encoding in balancing reconstruction quality with communication compression. a) GCE is computed using the SSIM values when the early stopping is satisfied and is averaged across clients. b) For CompJF and CompRandJF, the Tucker rank is $r=40$ and for Top-$k$, we use $k=30$. With this choice of hyperparameters, our methods and Top-$k$ have close percentage of compression in each communication round ($30\%$ and $34\%$)}.}
\label{figure:compression}
\end{figure}

Fig. \ref{figure:compression}a shows the percentage gain in GCE with respect to the CSR encoding and \ref{figure:compression}b highlights the total communication volume needed to achieve high quality reconstruction. The results in Fig. \ref{figure:compression} indicate that CompJF and CompRandJF lead to a substantial gain in GCE compared to full-size communication enhanced with CSR encoding and surpass Top-$k$. When $r=10$, the drop in SSIM overshadows the improvement in communication, thus the sharp decline in GCE.

\subsubsection{\textbf{Heterogeneous ranks}} Finally, we demonstrate the ability of CompRandJF, as well as CompJF, to adapt to heterogeneous ranks. This property is an improvement over other aggregation schemes that rely purely on averaging and require matching dimensions. In Fig. \ref{figure:heterocomm}, we use FullDecomp, adjusted for heterogeneous ranks, and FIRM as baselines for comparison. We observe that the performance of Comp JF and CompRandJF is consistent with the homogeneous case regarding reconstruction quality. 
\begin{figure}[h]
\centering
\begin{subfigure}{0.45\linewidth}
    \includegraphics[width=\linewidth]{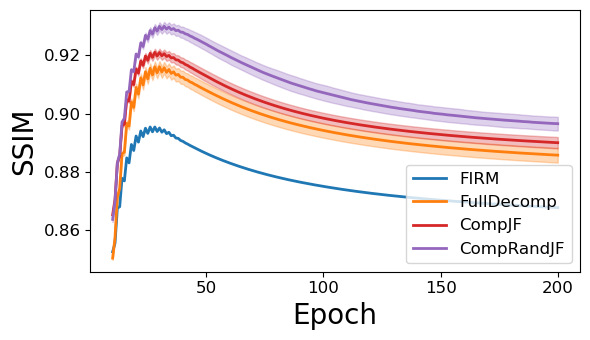}
    \caption{\footnotesize{Clients sample personalized ranks at the start.}}
\end{subfigure}
\begin{subfigure}{0.45\linewidth}
    \includegraphics[width=\linewidth]{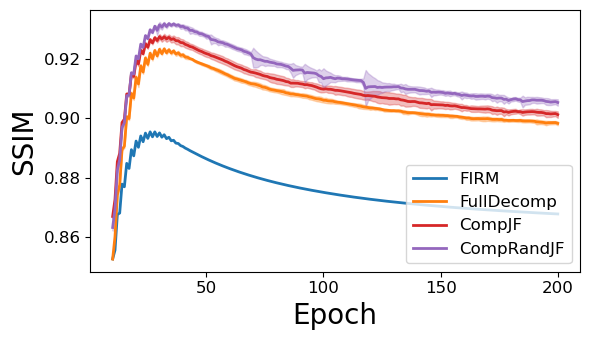}
    \caption{\footnotesize{Clients sample ranks before every communication round.}}
\end{subfigure}
\caption{\footnotesize{Performance of FullDecomp, CompJF and CompRandJF in heterogeneous rank setting for noisy data ($\sigma=0.1$). We conduct 10 independent simulations and plot the mean with one standard deviation.}}
\label{figure:heterocomm}
\end{figure}
\section{Conclusion}
In this paper, we present a federated image reconstruction method that leverages Tucker decomposition in conjunction with joint factorization and randomized sketching for efficient aggregation. We compare our approach against two benchmarks: 1) FIRM which addresses federated reconstruction without low-rank regularization and transmits high-dimensional data for each update, and 2) FullDecomp which is a common aggregation technique in tensor decomposition for federated learning that reconstructs the full weights for aggregation and optionally re-decomposes them to reduce downstream communication. Our method demonstrates superior reconstruction quality and more graceful performance degradation under noisy and undersampled conditions. In addition, we achieve better compression than magnitude-based sparsification (Top-\(k\)) and compressed sparse row encoding. By supporting both homogeneous and heterogeneous ranks, our approach delivers high-quality reconstruction across diverse scenarios. For future work, we plan to conduct convergence analyses, explore higher-dimensional datasets, and further investigate heterogeneous ranks—particularly relevant given that true data ranks are often unknown in practice, making randomized rank selection a valuable research avenue.
\bibliographystyle{plain}
\bibliography{reference}
\appendices
\section{Results for CompAVG}\label{appendix:moreresults}
We present the results of CompAVG, the naive extension of FedAVG in which factor matrices are averaged across clients. 
\begin{figure}[ht]
\centering
  \includegraphics[width=\linewidth]{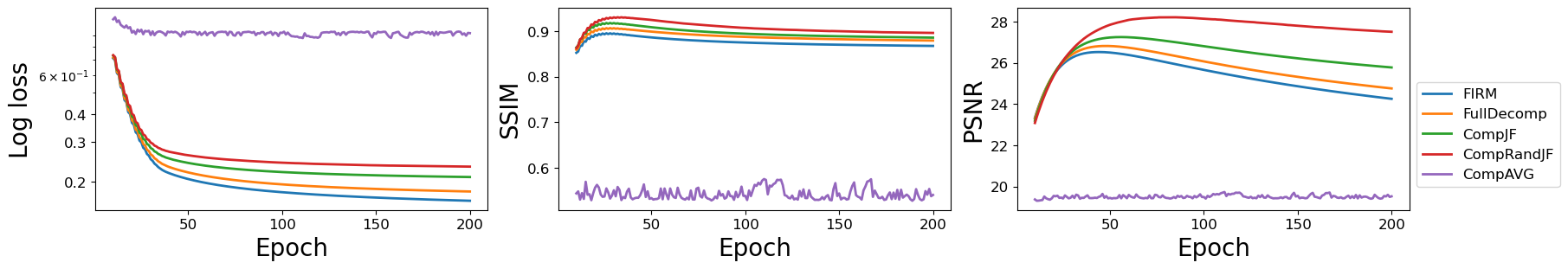}
  \caption{\footnotesize{Evaluating the performance of various methods given noisy data ($\sigma=0.1$). For CompJF, CompRandJF and FullDecomp, Tucker rank is $r=100$.}}
\label{appendix_fig:noisy_data_avg}
\end{figure}
Evidently in Fig. \ref{appendix_fig:noisy_data_avg}, CompAVG does not yield meaningful training compared to other methods. This is likely because the factor matrices are obtained via tensor decomposition after the full-size tensor update by gradient descent, unlike the weights in FedAVG which are directly optimized.
\end{document}